\documentclass[runningheads]{llncs}
\usepackage[T1]{fontenc}
\usepackage{graphicx}
\usepackage{graphicx,hyperref}
\usepackage{cite}
\usepackage{amsmath,amssymb,amsfonts}
\usepackage{textcomp}
\usepackage{xcolor}

\usepackage{url}
\usepackage{enumitem}
\usepackage{booktabs}
\usepackage{algorithm}
\usepackage{algpseudocode}
\usepackage{multirow}
\usepackage{comment}
\begin{document}
\title{NANSDE-Net: A NEURAL SDE FRAMEWORK FOR GENERATING TIME SERIES WITH MEMORY}
\titlerunning{NANSDE-Net}
\author{Hiromu Ozai\inst{1} \and
Kei Nakagawa\inst{2}}
\authorrunning{H. Ozai and K. Nakagawa}
\institute{Hiroshima University
\email{d251239@hiroshima-u.ac.jp}
\and
Graduate School of Business, Osaka Metropolitan University
\email{kei.nak.0315@gmail.com}}
\maketitle              % typeset the header of the contribution
\begin{abstract}
Modeling time series with long- or short-memory characteristics is a fundamental challenge in many scientific and engineering domains. While fractional Brownian motion has been widely used as a noise source to capture such memory effects, its incompatibility with It\^o calculus limits its applicability in neural stochastic differential equation~(SDE) frameworks. 
In this paper, we propose a novel class of noise, termed Neural Network-kernel ARMA-type noise~(NA-noise), which is an It\^o-process-based alternative capable of capturing both long- and short-memory behaviors. The kernel function defining the noise structure is parameterized via neural networks and decomposed into a product form to preserve the Markov property. Based on this noise process, we develop NANSDE-Net, a generative model that extends Neural SDEs by incorporating NA-noise. We prove the theoretical existence and uniqueness of the solution under mild conditions and derive an efficient backpropagation scheme for training. 
Empirical results on both synthetic and real-world datasets demonstrate that NANSDE-Net matches or outperforms existing models, including fractional SDE-Net, in reproducing long- and short-memory features of the data, while maintaining computational tractability within the It\^o calculus framework.

\end{abstract}
\section{Introduction}
%研究の重要性
Time series generation using deep learning has become an active area of research, particularly in modeling complex real-world dynamics.
This is especially true in fields such as physics and finance, where it is important to model continuous-time dynamics in a natural way. One promising approach is the Neural Ordinary Differential Equation~(Neural ODE) model~\cite{chen2018neural}, which combines differential equations with deep learning. Neural ODE views ResNet~\cite{he2016deep} as the limit of infinitely deep networks and models time evolution as a continuous process using ODEs.
A new family of Neural ODEs was proposed to model various dynamic systems, allowing for flexible representation of complex temporal patterns~\cite{luo2021deep,dang2023constrained}.

%先行研究を踏まえた課題
However, real-world time series often include random fluctuations due to measurement noise or external shocks. To address this, the Neural Stochastic Differential Equation~(Neural SDE) model has been proposed~\cite{kong2020sde,kidger2021neural}. It extends Neural ODE by introducing randomness using Brownian motion~(Bm), allowing the model to generate stochastic trajectories.
Building on this framework, several variants of Neural SDEs have been developed to model various stochastic systems~\cite{wu2022itowave,sun2024neural}.
Still, a limitation of Neural SDE is that it assumes the Bm as the noise process. The Bm is a Gaussian process with independent and stationary increments, but many real-world systems show long-~(resp. short-)memory or non-Markovian behavior\footnote{%
Long-memory refers to persistent statistical dependence across long time lags, in contrast to short-memory where correlations decay quickly. Non-Markovian behavior implies that the future evolution of a system depends not only on its current state but also on its entire history.
}, which Bm cannot model well. 
For example, in financial markets or natural phenomena, past events often influence future behavior for a long~(resp. short) time~\cite{cont2001empirical,strogatz2018nonlinear,gatheral2022volatility}.

To handle this, researchers have used fractional Brownian motion~(fBm)~\cite{hayashi2022fsde-net,nakagawa2024lf}, which can represent long-~(resp. short-)memory when the Hurst index $H >(\rm{resp}. <) 0.5$. 
However, fBm is not an It\^o process, so standard stochastic calculus methods cannot be applied. As a result, it is challenging to perform gradient-based optimization or derive closed-form representations for learning and inference.

To overcome this issue, \cite{inoue2005noise1,inoue2005noise2} proposed a class of AR($\infty$)-type Gaussian processes, also known as AutoRegressive and Moving Average~(ARMA)-type noise. 
These processes have stationary increments and memory, and they are also It\^o processes. This means they can model long-~(resp. short-)memory and still be analyzed using stochastic calculus. 
However, the main difficulty is that their kernel function $\ell(s,u)$, which determines the structure of the process, is hard to express explicitly in general.

%本研究
In this paper, we propose a new type of ARMA-type noise called Neural Network~(NN)-kernel ARMA-type noise~(NA-noise), where the kernel function $\ell(s,u)$ is approximated by neural networks. Specifically, we express the NN kernel as a product $\ell(s,u) = \ell_1(s)\ell_2(u)$ and use NNs to learn $\ell_1$ and $\ell_2$. This allows us to keep the process Markovian and make training easier.

We then build a new stochastic differential equation model driven by this NA-noise, which we call \textit{NANSDE (Neural Arma-type Noise SDE)}, and implement a generative model called \textit{NANSDE-Net}. 
Our NANSDE is written in the It\^o form, so it guarantees existence and uniqueness of the solution and can be solved numerically using methods like Euler--Maruyama. 
%We also derive a backpropagation algorithm for NANSDE, allowing us to train the model efficiently using gradients.

From a theoretical perspective, we review known results for explicit kernel forms~(Proposition~\ref{prop:ArmaItoProcess}) and extend them to NN kernels. 
We show that NANSDE-Net can reproduce both long- and short-memory behavior, similar to or better than fSDE-Net, while still being compatible with It\^o calculus. 
In experiments, NANSDE-Net performs well on both $H>0.5$ and $H<0.5$ cases, unlike traditional models that struggle with one or the other.
%We contribute a new approach to time series generation that combines memory effects and stochastic calculus in a unified and trainable framework. 

Our construction guarantees It\^o form and a Markov augmentation that facilitate both simulation and training; it does not, in general, enforce stationary increments of the learned noise nor guarantee a prescribed asymptotic decay of autocorrelations associated with rigorous long-memory theory. Designing neural parameterizations that simultaneously ensure Gaussianity, stationary increments, and targeted long-memory asymptotics remains an open direction. Likewise, broad comparisons with state-of-the-art sequence generators and architecture search are outside our present scope. The present work should therefore be read as an It\^o-compatible, proof-of-concept alternative to fBm-based approaches, clarifying when and how memory can be introduced into Neural SDEs without abandoning standard stochastic calculus.

%\subsection{Organization of this paper} 
%Section~\ref{sec_Prelim} provides the theoretical background of this study, including formal definitions of ARMA-type noise and stochastic differential equations (SDEs) driven by NN-kernel ARMA-type noise.  
%Section~\ref{sec_ATNSDE} introduces the proposed NANSDE-Net model, along with its architectural design and theoretical foundations.  
%Section~\ref{sec_generative} describes the implementation details of NANSDE-Net.  
%Section~\ref{sec_exp} presents numerical experiments on both synthetic and real-world datasets to evaluate the model's performance.  
%Finally, Section~\ref{sec_con} concludes the paper with a summary of findings and suggestions for future research.

%%%%%%%%%%%%%%%%%
%  Preliminary  %
%%%%%%%%%%%%%%%%%
\section{Preliminaries}\label{sec_Prelim}
This section introduces the noise process employed in this study.
In the works by \cite{inoue2005noise1} and \cite{inoue2005noise2}, a class of Gaussian processes was proposed that possesses stationary increments and memory, while also being semimartingales. 
fBm is commonly used as a Gaussian process with stationary increments. 
However, since fBm is not a semimartingale, standard tools of stochastic calculus cannot be applied, making it analytically intractable in many cases. 
To address this limitation, \cite{inoue2005noise1} and \cite{inoue2005noise2} introduced a new class of noise processes specifically designed to overcome such difficulties. 
Building upon this framework, the present study adopts a modified version of the noise process proposed in those works.

\subsection{Gaussian Processes with Stationary Increments and the Semimartingale Property}
We assume that the process $Z(\cdot)$ is a continuous process with stationary increments such that $Z(0) = 0$, and satisfies one of the following continuous-time AR($\infty$)-type equations:
\begin{equation} \label{eq:AR-infty-type}
\frac{dZ}{dt}(t) + \int_{-\infty}^{t} a(t-s) \frac{dZ}{dt}(s)ds = \frac{dW}{dt}(t).
\end{equation}
(see equation (2.10) and (2.17) in \cite{inoue2005noise1} for their precise formulation), where $\{W(t)\}_{t \in \mathbb{R}}$ is a one-dimensional standard Bm defined on $(\Omega, \mathcal{F}, P)$ satisfying $W(0)=0$, and $dZ/dt$ and $dW/dt$ are the derivatives of $Z(\cdot)$ and $W(\cdot)$ respectively in the random distribution sense. 
The kernel $a(\cdot)$ is a nonnegative decreasing function with some adequate conditions to be specified in \cite{inoue2005noise1}. The case $a(\cdot)=0$ yields the usual white noise.

As is clear from the definition, $\{Z(t)\}_{t \in \mathbb{R}}$ does not have independent increments. In other words, it is a stochastic process with memory. 
Under appropriate conditions related to $a(\cdot)$~(see \cite{inoue2005noise1,inoue2005noise2}), $Z(\cdot)$ has the following MA($\infty$)-type representation. 
In the proposition below, a short range memory type kernel is used for $a(\cdot)$.

\begin{proposition}[\cite{inoue2005noise1,inoue2005noise2}]
For real parameter $q \in \mathbb R$ and positive parameter $p \in \mathbb R_{>0}$ that satisfy $p > q$, determine $a:(0,\infty) \to \mathbb R$ in~(\ref{eq:AR-infty-type}) as follows:
\begin{equation} \label{eq:Noise_kernel_a}
a(t) := q e^{-(p-q)t}.
\end{equation}
Then, $\{Z(t)\}_{t \in \mathbb R}$ in~(\ref{eq:AR-infty-type}) has the following MA($\infty$)-type representation:
\begin{equation}
\label{eq:Arma_type_Noise}
Z(t) = W(t) - \int_{0}^{t} \left\{ \int_{-\infty}^{s} c(s-u) dW(u) \right\} ds, \quad t \in \mathbb R
\end{equation}
Where $c:(0,\infty) \to \mathbb R$ is expressed in the following form:
\begin{equation} \label{eq:Noise_kernel_c}
c(t) = q e^{-pt}
\end{equation}
%And $\{W(t)\}_{t \in \mathbb R}$ is a one-dimensional Wiener process defined on $(\Omega, \mathcal{F}, P)$ satisfying $W(0)=0$.
\end{proposition}

%Next, we transform $\{Z(t)\}_{t \in \mathbb{R}}$ according to the general theory of filtering~(see \cite{Liptser-Shiryaev2001}). First, we define the filtration $\{\mathcal{F}_{t}\}_{t \geq 0}$ as follows:
%\begin{equation}
%    \mathcal{F}_{t} := \sigma(Z(u):0 \leq u \leq t) \lor \mathcal{N}, \quad t \in \mathbb{R}
%\end{equation}
%where $\mathcal{N}$ is the set consisting of all null sets.
%The innovation process $\{\hat{W}(t)\}_{t \geq 0}$ of $\{Z(t)\}_{t \in \mathbb{R}}$ is defined as follows:
%\begin{equation}
%    \hat{W}(t) = Z(t) + \int_{0}^{t} E\left[ \int_{-\infty}^{s} c(s-u) dW(u) \middle| \mathcal{F}_s \right] ds, \quad t \geq 0,
%\end{equation}
%$\{\hat{W}\}_{t \geq 0}$ becomes a one-dimensional Wiener process, and $\{\mathcal{F}_{t}\}_{t \geq 0}$ coincides with the filtration generated from $\{\hat{W}\}_{t \geq 0}$. 
%In other words, $\{\mathcal{F}_{t}\}_{t \geq 0} = \{\hat{\mathcal{F}}_{t}\}_{t \geq 0} := \sigma(\hat{W}(u):0 \leq u \leq t) \lor \mathcal{N}$.
%Furthermore, by the deterministic function $\ell(s,u)$ and the innovation process $\{\hat{W}\}_{t \geq 0}$, $\{Z(t)\}_{t \geq 0}$ has the following representation of the It\^o process:
%\begin{equation} \label{eq:ItoProcessRepOfZ}
%    Z(t) = \hat{W}(t) - \int_{0}^{t} \left\{ \int_{0}^{s} \ell(s,u) d\hat{W}(u) \right\} ds, \quad t \geq 0.
%\end{equation}

Furthermore, by transforming it according to the filtering theory, we obtain the representation of the Ito process of $\{Z(t)\}_{t \geq 0}$ shown below.
\begin{equation} \label{eq:ItoProcessRepOfZ}
    Z(t) = \hat{W}(t) - \int_{0}^{t} \left\{ \int_{0}^{s} \ell(s,u) d\hat{W}(u) \right\} ds, \quad t \geq 0.
\end{equation}
where $\ell(s,u)$ is a deterministic function and $\{\hat{W}(t)\}_{t \geq 0}$ denotes the innovation process defined below, which is a one-dimensional Brownian motion.
\begin{equation}
    \hat{W}(t) = Z(t) + \int_{0}^{t} E\left[ \int_{-\infty}^{s} c(s-u) dW(u) \middle| \mathcal{F}_s \right] ds, \quad t \geq 0,
\end{equation}

To use this $\{Z(t)\}_{t \geq 0}$ as noise, we need to find the specific form of the Volterra kernel $\ell(s,u)$.
The explicit form of this kernel was made available for the first time thanks to the results of \cite{inoue2005noise1} and \cite{inoue2005noise2}.
The explicit form of $\ell(s,u)$ can be obtained from the following proposition.

\begin{proposition}[\cite{inoue2005noise1,inoue2005noise2}] \label{prop:ArmaItoProcess}
We take $c(\cdot)$ in the form of equation \eqref{eq:Noise_kernel_c}.
Then $\ell(s,u)$ in (\ref{eq:ItoProcessRepOfZ}) has the following closed-form representation:
\begin{equation} \label{eq:Arma-type-noise-kernel}
\begin{aligned}
    \ell(s,u) &= e^{-ps}\ell(u), \qquad s>u>0, \\
    \ell(u) &:= q e^{pu} \left\{ 1-\frac{2q(p-q)}{(2p-q)^2 e^{2(p-q)u} - q^2} \right\}, \quad u>0.
\end{aligned}
\end{equation}
\end{proposition}
By explicitly determining the shape of $\ell$ in this way, the It\^o process representation of $\{Z(t)\}_{t \geq 0}$ can be specifically determined, making it possible to treat $\{Z(t)\}_{t \geq 0}$ as driving noise within the scope of It\^o calculus. 
We refer to this type of noise as ARMA-type noise.
Furthermore, although $\{Z(t)\}_{t \geq 0}$ itself is not a Markov process, it can be embedded as a component of a Markov process because $\ell(s,u)$ can be decomposed into functions of $s$ and $u$.
That is, if $K(\cdot)$ is defined as follows:
\begin{equation}
    K(t) := \int_{0}^{t} \ell(s) d\hat{W}(s), \quad t \geq 0,
\end{equation}
$(Z(t),K(t))^{\top}$ is described as a solution to the following Markov-type SDE \cite{inoue-moriuch-nakamura}:
\begin{equation}
    \left\{ \,
    \begin{aligned}
    & dZ(t) = -e^{-pt}K(t)dt + d\hat{W}(t), \\
    & dK(t) = \ell(t)d\hat{W}(t).
    \end{aligned}
\right.
\end{equation}
This makes it possible to use numerical calculation methods, such as the Euler--Maruyama method.

\subsection{ARMA-Type noise with the kernel approximated by a neural network}
The kernel in equation \eqref{eq:Noise_kernel_a} has been shown to represent a short-memory structure~(see \cite{inoue2005noise1}). 
To reproduce long-memory data, it is necessary to employ a kernel that exhibits long range memory characteristics. \cite{inoue2005noise1} introduces such a kernel in the form $a(t)=\frac{p}{(t+1)^{p+1}}, t>0$. 
However, an explicit expression for the corresponding $\ell(s,u)$ in this case has not yet been obtained, which remains an open problem. 
Therefore, in this paper, we adopt a method that approximates the kernel using neural networks in order to capture long-memory behavior. 
To ensure the Markov property, we adopt the form $\ell(s,u)=\ell_1(s)\ell_2(u)$.
The noise process used in this study takes the following form.
\begin{equation} \label{eq:NA-noise}
    \left\{ \,
    \begin{aligned}
    & dZ(t) = -\ell_1(t)K(t)dt + dW(t), \\
    & dK(t) = \ell_2(t)dW(t).
    \end{aligned}
\right.
\end{equation}
Where $\ell_1$ and $\ell_2$ are deterministic functions represented by neural networks and $W$ is a one-dimensional Wiener process. 
While $Z$ is indeed a Gaussian process, its increments are not guaranteed to be stationary.
We refer to this type of noise as NN-kernel ARMA-Type noise~(abbreviated as NA-noise).

\subsection{SDE Driven by NN-kernel ARMA-Type Noise}
To implement NN-kernel ARMA-Type Noise SDE-Net~(abbreviated as NANSDE-Net), we study SDEs where the standard Bm is replaced by NA-noise. 
We consider SDEs driven by NA-noise \eqref{eq:NA-noise} of the form
\begin{equation}
\label{eq:NANSDE}
\left\{ \,
\begin{aligned}
 X_t &= X_0 + \int_0^t b (s, X_s ) ds + \int_0^t \sigma (s, X_s) dZ(s) \\
     &= X_0 + \int_0^t \{b (s, X_s ) - \ell_1(s) \sigma(s, X_s) K(s)\} ds \\
     &\qquad\qquad\qquad + \int_0^t \sigma (s, X_s) dW(s), \\
K_t &= \int_0^t \ell_2(s) dW(s),
\end{aligned}
\right.
\end{equation}
where $W(t)$ is standard Bm.
Note that when $\ell_2(u)$ is $0$, NANSDE-Net becomes SDE-Net. Therefore, NANSDE-Net can be regarded as an extension of SDE-Net. Furthermore, \eqref{eq:NANSDE} is an It\^o process, the results of conventional stochastic calculus can be applied to determine the existence and uniqueness of solutions. Furthermore, the Euler--Maruyama method can be used for numerical calculations.

\section{Generative Modeling of Time Series Using NANSDE-Net}\label{sec_generative}
In \cite{hayashi2022fsde-net}, to compare the ability of different noise processes to reproduce long-memory effects, the authors deliberately limit the expressive power of the drift and diffusion terms. 
Specifically, time $t$ is not used as an input to the functions $b$ and $\sigma$. 
To conduct experiments under the same setting as in \cite{hayashi2022fsde-net}, in the following, we assume that the time evolution of a stochastic process $X = \{X_t\}_{t \ge 0}$ is described by the following NANSDE 
\begin{equation}
\label{eq:no_time_ANTSDE_stock}
\left\{ \,
    \begin{aligned}
    & dX_t = \{b(X_t, \theta) - \ell_1(t,\theta) \sigma(X_t, \theta) K(t,\theta)\} dt\\
    &\qquad\qquad\qquad + \sigma (X_t, \theta) dW_t, \\
    & dK(t,\theta) = \ell_2(t,\theta)dW_t.
    \end{aligned}
\right.
\end{equation}
where the drift term $b$, the diffusion term $\sigma$, the kernel $\ell_1$ and $\ell_2$ are parameterized by NN.
Hereafter, we optimize $b$, $\sigma$, $\ell_1$ and $\ell_2$ with respect to the NN-parameters $\theta$. 

\subsection{Network Architecture}
To estimate drift, volatility and kernel functions of \eqref{eq:no_time_ANTSDE_stock} by neural networks, we use $L$-layer multi-layer perceptron~(MLP) as network architecture\footnote{We use merely MLP for simplicity. We note that we can expect the improvement of fitting by optimizing network architecture.}:
\[
h^{\ell}_i = \sum_{j=1}^{N_{\ell-1}} w_{ij}^{\ell-1} x^{\ell-1}_j + b^{\ell-1}_i , \quad 
x^{\ell}_i = \varphi(h^\ell_i) 
\]
for each $\ell = 1,\ldots, L$ and each $i = 1, \ldots, N_\ell$. In this case $x^0$ and $x^L$ be input and output data and NN function is such that $x^L = f_\theta (x^0)$ where $\theta = \{ w_{ij}^\ell, b^\ell_i \}_{i, j, \ell}$. 
In the sequel, let $\Theta$ be the space of NN parameters on which $\theta$ takes values. Then existence and uniqueness of NANSDE, and precision assurance of numerical solutions are established as follows. 

\begin{theorem}[Informal]
The NANSDE-Net generator \eqref{eq:no_time_ANTSDE_stock} whose network architecture is given by MLP with $\mathrm{tanh}$ activation function has a unique solution, which can be numerically solved by the explicit Euler scheme.  
\end{theorem}
\begin{proof}
Note that the activation function $\mathrm{tanh}$ is smooth and its derivatives of any order are bounded. 
Thus, noting that the composition of functions preserves regularity and boundedness, we conclude that the NANSDE \eqref{eq:no_time_ANTSDE_stock}, whose drift, volatility, and kernel are parameterized by multilayer perceptrons $f_\theta$, satisfies the Lipschitz continuity and linear growth conditions. Therefore, the solution exists uniquely, and the Euler--Maruyama method is applicable.
\end{proof}

\subsection{Algorithm}
The generation procedure of NANSDE-Net is summarized in Algorithm \ref{alg:fsdenet}.
This algorithm follows the same procedure as Algorithm in Section V of \cite{hayashi2022fsde-net}. For further details, please refer to \cite{hayashi2022fsde-net}.

\begin{algorithm}[t]
\caption{Optimization of NANSDE-Net generator}\label{alg:fsdenet}
\begin{algorithmic}
\Require $\{ X_0,\ldots, X_T \}$ - a realized path of a stochastic process until time horizon $T$, sample size $M$, learning rate $\eta$, number of optimization steps $k$
\Ensure $(\theta_b, \theta_\sigma, \theta_{\ell_1}, \theta_{\ell_2})$ - the optimized parameter for NANSDE-Net generator
\While{not converged}
\For{$k$ steps}
    \State Let $\{ \hat{X}^{(i)}_0, \ldots \hat{X}^{(i)}_T \}_{i=1}^M $ be the realization of generated paths of size $M$ and let $p_t(\theta, \cdot)$ be the probability density function of log-difference process, which is estimated from generated values $\{ \hat{r}^{(i)}_t \}_{i=1}^M$ where $\hat{r}^{(i)}_t = \log (\hat{X}^{(i)}_{t+1} / \hat{X}^{(i)}_t)$. 
    \State Compute the gradient of loss function
    \[
    \mathcal{L}(\theta) = - \frac{1}{T} \sum_{t=0}^T \log p_{\theta} (t, r_t).
    \]
    \State Descent parameters: $
    \theta \gets \theta - \eta \nabla_{\theta} \mathcal{L} (\theta)$. 
\EndFor
\EndWhile
\end{algorithmic}
\end{algorithm}
\section{Experiments}\label{sec_exp}
\if0
\subsection{Dataset}
We use the following synthetic and real time-series data as input, from which we will generate new paths with the help of methods using neural networks. 
For the purpose of comparing performance, the following data set is the same as that used in Paper \cite{hayashi2022fsde-net}.
For detailed information on the data, the reader is referred to \cite{hayashi2022fsde-net}.
\subsubsection{Synthetic data} 
As a simple example of a series appropriate for the current situation, we generate fractional Brownian motion with low Hurst indices artificially and use them as data.
\begin{itemize}
%\item fBm: Fractional Brownian motion data with Hurst indices of 0.2, 0.3, and 0.4.
\item fBm: Fractional Brownian motion data with Hurst indices of 0.2, 0.3.
\end{itemize}

A realization of each process is obtained by dividing the interval $[0,1]$ into 1000 pieces with the common width. 
\fi

We evaluate our models on synthetic and real time series, and generate sample paths from each trained model. 
For fair comparison, we follow the dataset choices and settings of \cite{hayashi2022fsde-net}. 
Readers can consult \cite{hayashi2022fsde-net} for additional dataset details.

\subsubsection{Synthetic data}
We simulate fractional Brownian motion (\textbf{fBm}) with Hurst indices $H\in\{0.2,\,0.3\}$ to represent short-memory (anti-persistent) dynamics. 
Each path is discretized over $[0,1]$ into 1000 equal steps ($\Delta t=10^{-3}$).

\subsubsection{Real data}
We evaluate on six real-world time series that cover both long- and short-memory behaviors:
\textbf{SPX} (S\&P 500), \textbf{TPX} (TOPIX), and \textbf{SX5E} (Euro Stoxx 50) daily closing prices;
\textbf{NileMin} (annual minimum Nile levels at the Roda gauge, 622–1281);
\textbf{ethernetTraffic} (LAN traffic at Bellcore); and
\textbf{NhemiTemp} (monthly Northern Hemisphere temperature anomalies, 1854–1989, CRU/UEA).
The equity index series are obtained from Bloomberg. 
The other datasets are from the CRAN package \texttt{longmemo} (\url{https://cran.r-project.org/web/packages/longmemo/index.html}), which also provides detailed descriptions.

\if0
\subsubsection{Real data}
Also we apply our generation method for time-series data extracted from real world (SPX, TPX, SX5E, NileMin, ethernetTraffic, NhemiTemp) which exhibits long- or short-memory. 
SPX, TPX, SX5E are Daily closing prices of the S\&P 500 index, TOPIX index, and Euro Stoox 50 index.
NileMin: The yearly minimum water levels of the Nile River from 622 to 1281, as measured at the Roda gauge near Cairo, are used.
ethernetTraffic: Ethernet traffic data collected from a LAN at Bellcore, Morristown.

NhemiTemp: Monthly temperature data for the Northern Hemisphere from 1854 to 1989, obtained from the database maintained by the Climate Research Unit at the University of East Anglia, Norwich, England.
Each value represents the temperature difference (in degrees Celsius) from the corresponding monthly average over the reference period 1950–1979.

For the SPX,TPX, and SX5E series, we obtained from Bloomberg terminal.
The other real-world time series are obtained from the CRAN~(Comprehensive R Archive Network) website.
They can be downloaded from
\url{https://cran.r-project.org/web/packages/longmemo/index.html},
where detailed descriptions of each dataset are also available.
\fi

\begin{table*}[t]
  \caption{Numerical results for each performance metric and generative model for fBm and other types of data from the real world. Bold indicates the best performance.}
  \label{tab:result}
  \centering
  \scriptsize
\resizebox{\linewidth}{!}{
\begin{tabular}{ccccccc}
Data & Method & Hurst Index & Marginal Distribution & ACF & Weighted ACF & $R^2$ Score\\ 
\hline \hline
\multirow{4}{*}{fBm(H=0.2)} 
& Original & 0.2 (True value) & - & - & - & - \\
& RNN & 0.464 $\pm$ 0.111 & \bf{0.615 $\pm$ 0.024} & \bf{1.215} & \bf{0.633} & -1.086 $\pm$ 0.180 \\  
& SDE & 0.463 $\pm$ 0.132 & 1.058 $\pm$ 0.024 & 1.220 & 0.639 & -0.121 $\pm$ 0.044 \\ 
& fSDE & 0.581 $\pm$ 0.144 & 1.518 $\pm$ 0.013 & 1.220 & 0.638 & \bf{-0.017 $\pm$ 0.018} \\ 
& NANSDE & \bf{0.453 $\pm$ 0.133} & 1.094 $\pm$ 0.070 & 1.276 & 0.727 & -0.109 $\pm$ 0.049 \\
\hline
\multirow{4}{*}{fBm(H=0.3)} 
& Original & 0.3 (True value) & - & - & - & - \\
& RNN & 0.463 $\pm$ 0.121 & \bf{0.475 $\pm$ 0.024} & \bf{1.217} & \bf{0.666} & -1.139 $\pm$ 0.168 \\  
& SDE & 0.489 $\pm$ 0.157 & 0.567 $\pm$ 0.199 & 1.995 & 1.279 & -0.513 $\pm$ 0.237 \\ 
& fSDE & 0.581 $\pm$ 0.130 & 1.186 $\pm$ 0.069 & 1.429 & 0.845 & \bf{-0.069 $\pm$ 0.037} \\ 
& NANSDE & \bf{0.422 $\pm$ 0.177} & 0.561 $\pm$ 0.191 & 1.995 & 1.415 & -0.716 $\pm$ 0.320 \\
\hline
%\multirow{4}{*}{fBm(H=0.4)} 
%& Original & 0.4 (True value) & - & - & - & - \\
%& RNN & 0.453 $\pm$ 0.119 & \bf{0.256 $\pm$ 0.022} & \bf{1.210} & \bf{0.650} & -0.876 $\pm$ 0.160 \\  
%& SDE & 0.290 $\pm$ 0.218 & 0.404 $\pm$ 0.061 & 2.331 & 1.735 & -2.836 $\pm$ 0.998 \\ 
%& fSDE & 0.574 $\pm$ 0.172 & 0.650 $\pm$ 0.055 & 1.273 & 0.750 & \bf{-0.269 $\pm$ 0.084} \\ 
%& NANSDE & \bf{0.453 $\pm$ 0.113} & 0.262 $\pm$ 0.024 & 1.215 & 0.661 & -0.923 $\pm$ 0.202 \\
%\hline
\multirow{4}{*}{SPX} 
& Original & 0.614 & - & - & - & - \\
& RNN & 0.471 $\pm$ 0.097 & 0.704 $\pm$ 0.011 & 2.717 & \bf{1.340} & -6.376 $\pm$ 0.605 \\
& SDE & 0.510 $\pm$ 0.117 & 0.204 $\pm$ 0.045 & 2.762 & 1.430 & -2.362 $\pm$ 0.946 \\
& fSDE & \bf{0.591 $\pm$ 0.112} & 0.248 $\pm$ 0.043 & \bf{2.716} & 1.343 & \bf{-1.170 $\pm$ 0.293} \\
& NANSDE & 0.472 $\pm$ 0.115 & \bf{0.202 $\pm$ 0.058} & 2.804 & 1.404 & -1.481 $\pm$ 0.575 \\
\hline
\multirow{4}{*}{TPX} 
& Original & 0.396 & - & - & - & - \\
& RNN &  0.455 $\pm$ 0.098 & 0.327 $\pm$ 0.012 & \bf{1.889} & \bf{0.829} &  -1.524 $\pm$ 0.183 \\
& SDE & 0.451 $\pm$ 0.098 & 0.223 $\pm$ 0.024 & 1.900 & 0.842 & -3.471 $\pm$ 0.832 \\
& fSDE & 0.590 $\pm$ 0.137 & 0.261 $\pm$ 0.018 & 1.894 & 0.833 & \bf{-1.053 $\pm$ 0.249}
\\
& NANSDE & \bf{0.396 $\pm$ 0.139} & \bf{0.152 $\pm$ 0.038} & 2.018 & 0.972 & -2.371 $\pm$ 0.653 \\
\hline
\multirow{4}{*}{SX5E} 
& Original & 0.328 & - & - & - & - \\
& RNN &  0.468 $\pm$ 0.102 &  0.744 $\pm$ 0.011 & \bf{2.479} & \bf{1.269} & -2.940 $\pm$ 0.277 \\
& SDE & 0.464 $\pm$ 0.121 &  \bf{0.277 $\pm$ 0.042} & 2.609 & 1.343 & -5.410 $\pm$ 1.291 \\
& fSDE & 0.575 $\pm$ 0.130 & 0.277 $\pm$ 0.128 & 2.893 & 1.514 & -1.616 $\pm$ 0.435 \\
& NANSDE & \bf{0.449 $\pm$ 0.130} & 0.351 $\pm$ 0.168 & 3.018 & 1.717 & \bf{-1.606 $\pm$ 0.582} \\
\hline
\multirow{4}{*}{NileMin} 
& Original & 0.973 & - & - & - & - \\
& RNN & 0.477 $\pm$ 0.121  & 0.554 $\pm$ 0.040 & \bf{1.482} & \bf{0.948}  & -4.900 $\pm$ 0.702 \\
& SDE & 0.455 $\pm$ 0.145  & 1.211 $\pm$ 0.102 & 1.527 & 0.997  & -29.771 $\pm$ 8.093 \\
& fSDE & 0.646 $\pm$ 0.151 & \bf{0.233 $\pm$ 0.051} & 1.511 & 0.967  & \bf{-1.543 $\pm$ 0.410} \\
& NANSDE & \bf{0.952 $\pm$ 0.182} & 0.435 $\pm$ 0.192 & 2.938 & 2.463 & -2.837 $\pm$ 3.714 \\
\hline
\multirow{4}{*}{ethernetTraffic} 
& Original & 0.750 & - & - & - & - \\
& RNN & 0.442 $\pm$ 0.100 & 1.400 $\pm$ 0.009 & \bf{1.646} & \bf{1.124} & -9.297 $\pm$ 0.398 \\
& SDE & 0.668 $\pm$ 0.138  & 0.980 $\pm$ 0.043 & 3.902 & 2.562 & -0.753 $\pm$ 0.320 \\
& fSDE & \bf{0.761 $\pm$ 0.134} & 0.931 $\pm$ 0.218 & 3.485 & 2.577 & -0.880 $\pm$ 0.508 \\
& NANSDE & 0.677 $\pm$ 0.135 & \bf{0.801 $\pm$ 0.037} & 3.224 & 1.827 & \bf{-0.739 $\pm$ 0.148} \\
\hline
%\multirow{4}{*}{NBSdiff1kg} 
%& Original & 0.666 & - & - & - & - \\
%& RNN & 0.483 $\pm$ 0.177  & 0.492 $\pm$ 0.046 & \bf{1.105} & \bf{0.476} & -0.945 $\pm$ 0.299 \\
%& SDE & 0.661 $\pm$ 0.198  & \bf{0.490 $\pm$ 0.042} & 1.158 & 0.503 & -0.668 $\pm$ 0.381 \\
%& fSDE & \bf{0.671 $\pm$ 0.181}  & 0.526 $\pm$ 0.067 & 1.175 & 0.551 & \bf{-0.536 $\pm$ 0.318} \\
%& NANSDE & 0.557 $\pm$ 0.199 & 0.507 $\pm$ 0.055 & 1.200 & 0.551 & -0.663 $\pm$ 0.510 \\
%\hline
\multirow{4}{*}{NhemiTemp} 
& Original & 1.066 & - & - & - & - \\
& RNN & 0.463 $\pm$ 0.120 & 0.652 $\pm$ 0.023 & \bf{2.069} & \bf{1.343} & -6.973 $\pm$ 0.614 \\
& SDE & 0.603 $\pm$ 0.075 & \bf{0.172 $\pm$ 0.023} & 2.092 & 1.348 & -2.290 $\pm$ 0.296 \\
& fSDE & 0.599 $\pm$ 0.161 & 0.965 $\pm$ 0.187 & 2.575 & 1.890 & -17.534 $\pm$ 6.263 \\
& NANSDE & \bf{1.025 $\pm$ 0.235} & 0.455 $\pm$ 0.239 & 4.804 & 3.258 & \bf{-1.677 $\pm$ 2.245} \\
\hline
\end{tabular}
}
\end{table*}

\subsection{Experimental Setup}
In the following, we quantitatively evaluate the performance of the NANSDE-Net using a specified criterion.
For comparison, existing time series generators—including recurrent neural networks (RNN), SDE-Net, and fSDE-Net with $H>1/2$ are used in place of NANSDE-Net.
For the network architecture, we employ a two-layer MLP with 20 hidden units for SDE-Net, fSDE-Net, and NANSDE-Net, while a vanilla RNN with 40 hidden units is used for the RNN generator.

\begin{figure}[t]
\begin{minipage}[b]{0.32\linewidth}
    \centering
    \includegraphics[keepaspectratio, scale=0.2]{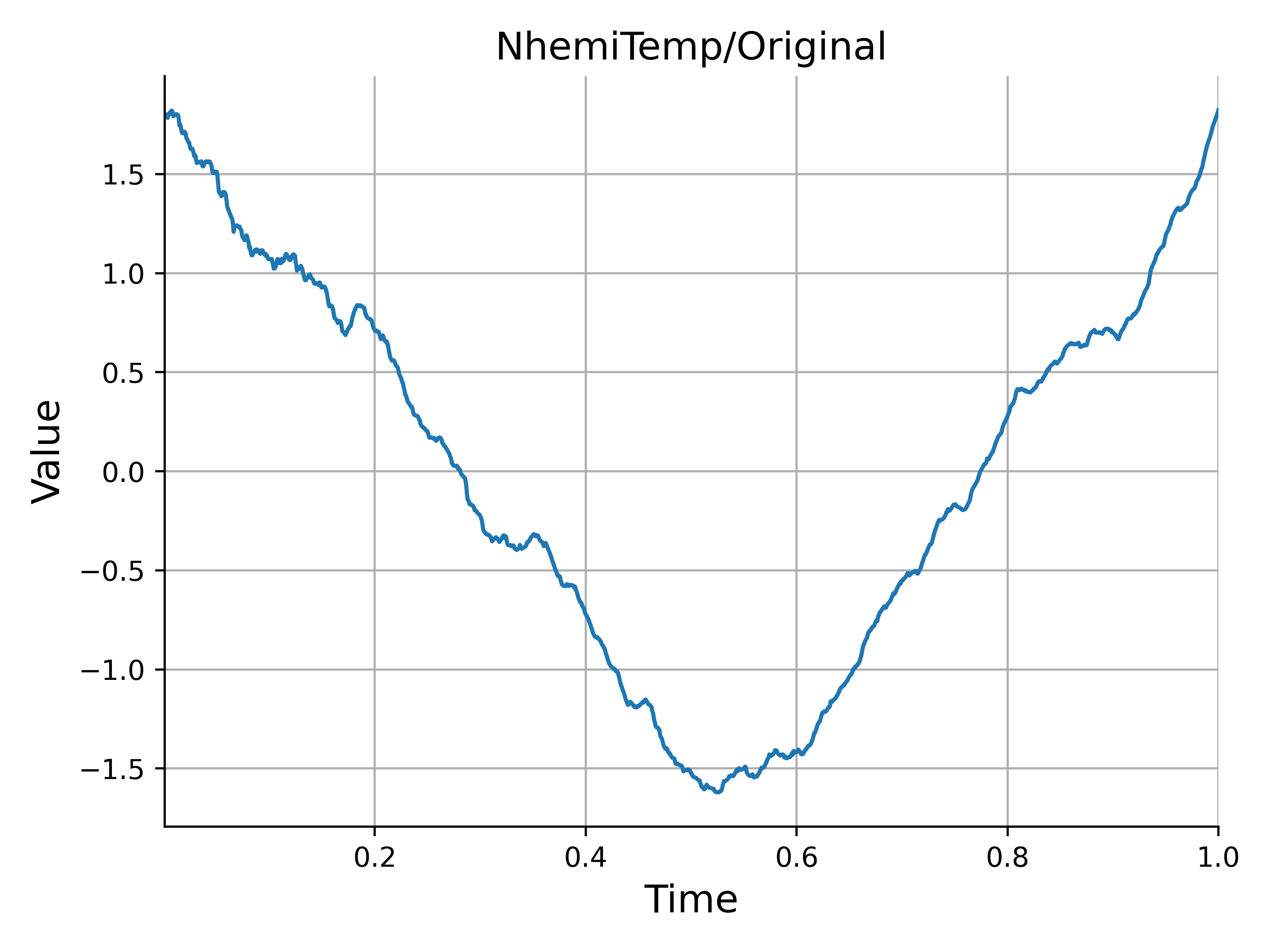}
    %\subcaption{}
\end{minipage}
\begin{minipage}[b]{0.32\linewidth}
    \centering
    \includegraphics[keepaspectratio, scale=0.2]{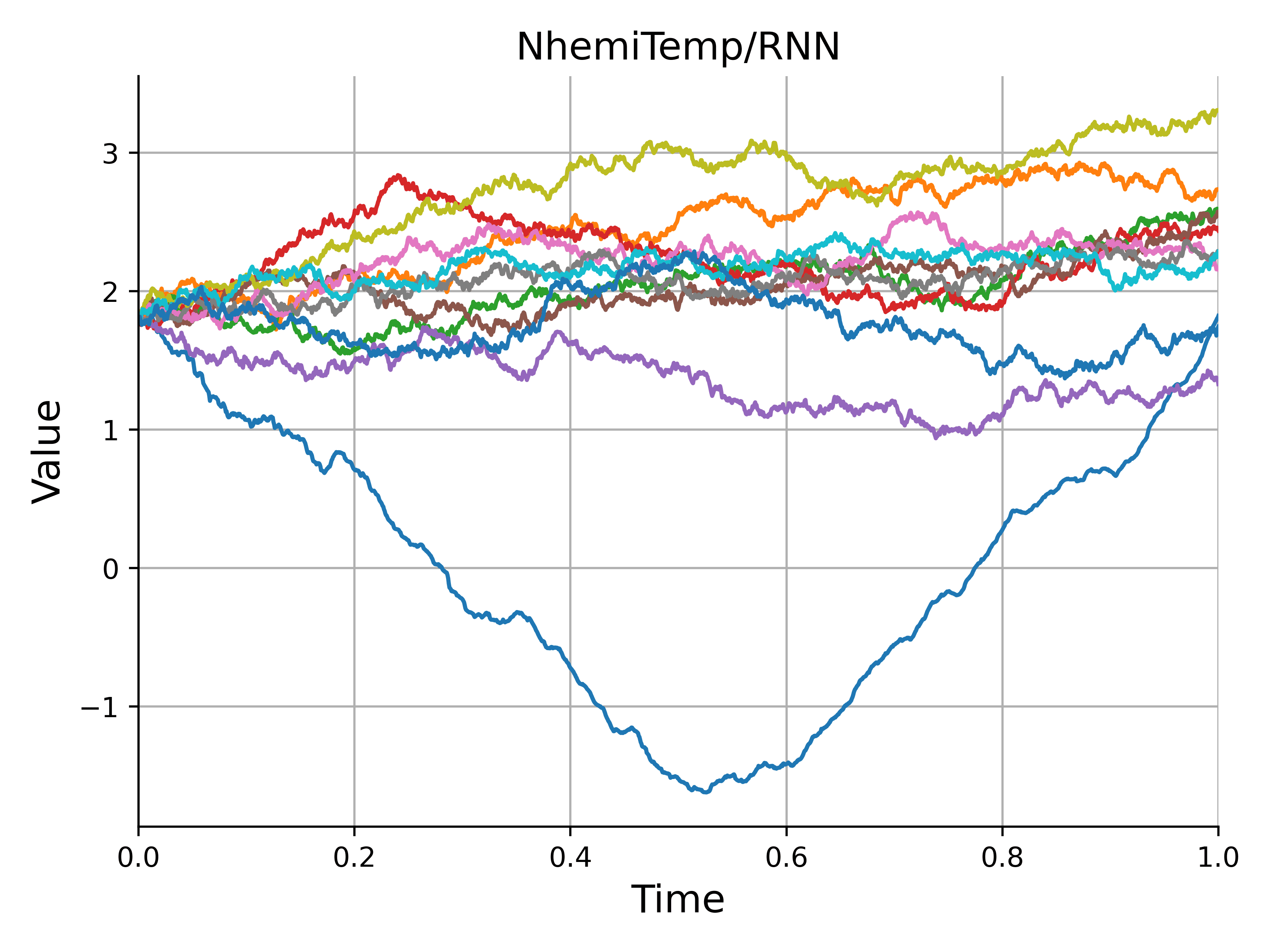}
    %\subcaption{}
\end{minipage}
\begin{minipage}[b]{0.32\linewidth}
    \centering
    \includegraphics[keepaspectratio, scale=0.2]{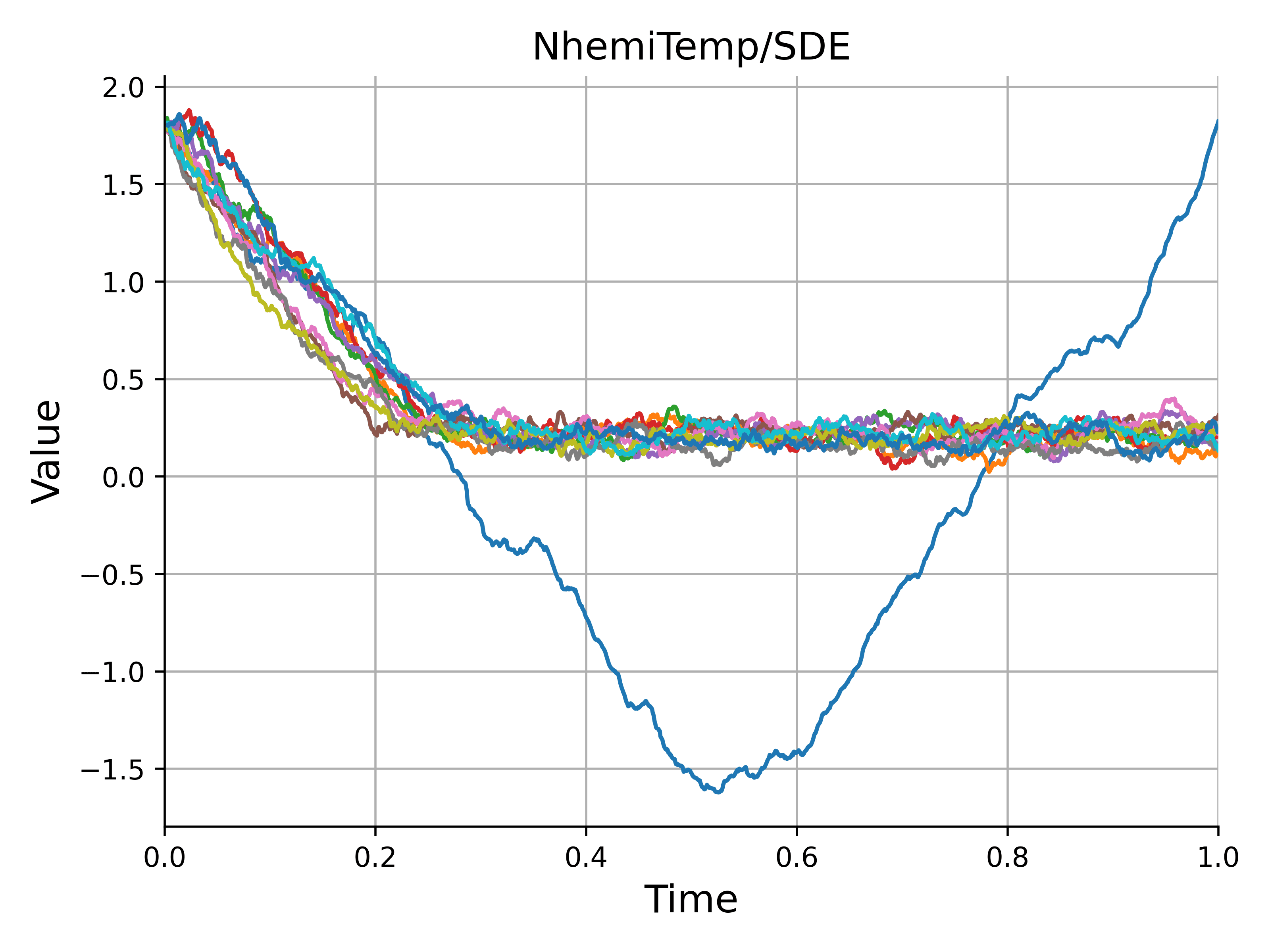}
\end{minipage}\\
\begin{minipage}[b]{0.32\linewidth}
    \centering
    \includegraphics[keepaspectratio, scale=0.2]{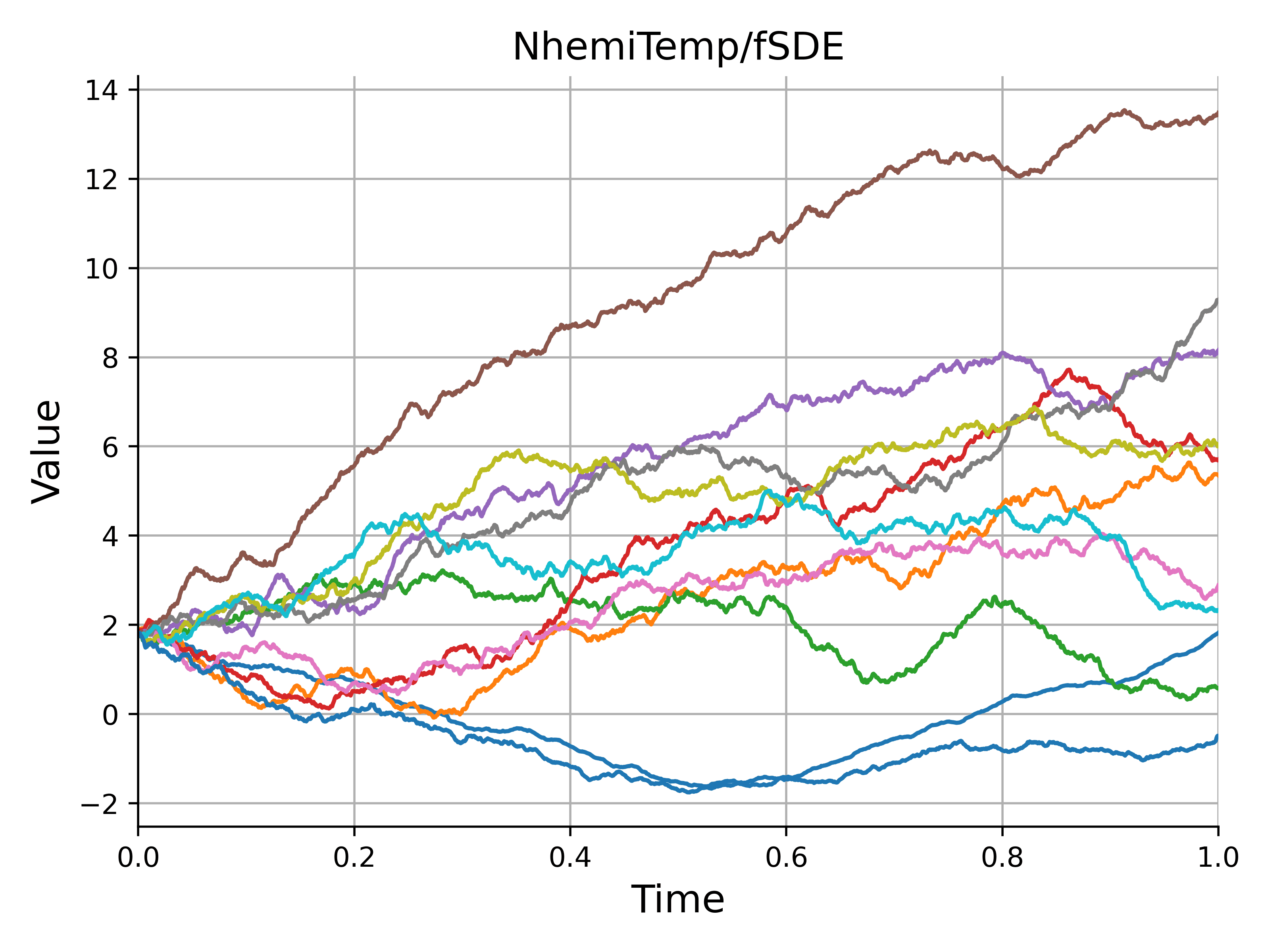}
\end{minipage}
\begin{minipage}[b]{0.32\linewidth}
    \centering
    \includegraphics[keepaspectratio, scale=0.2]{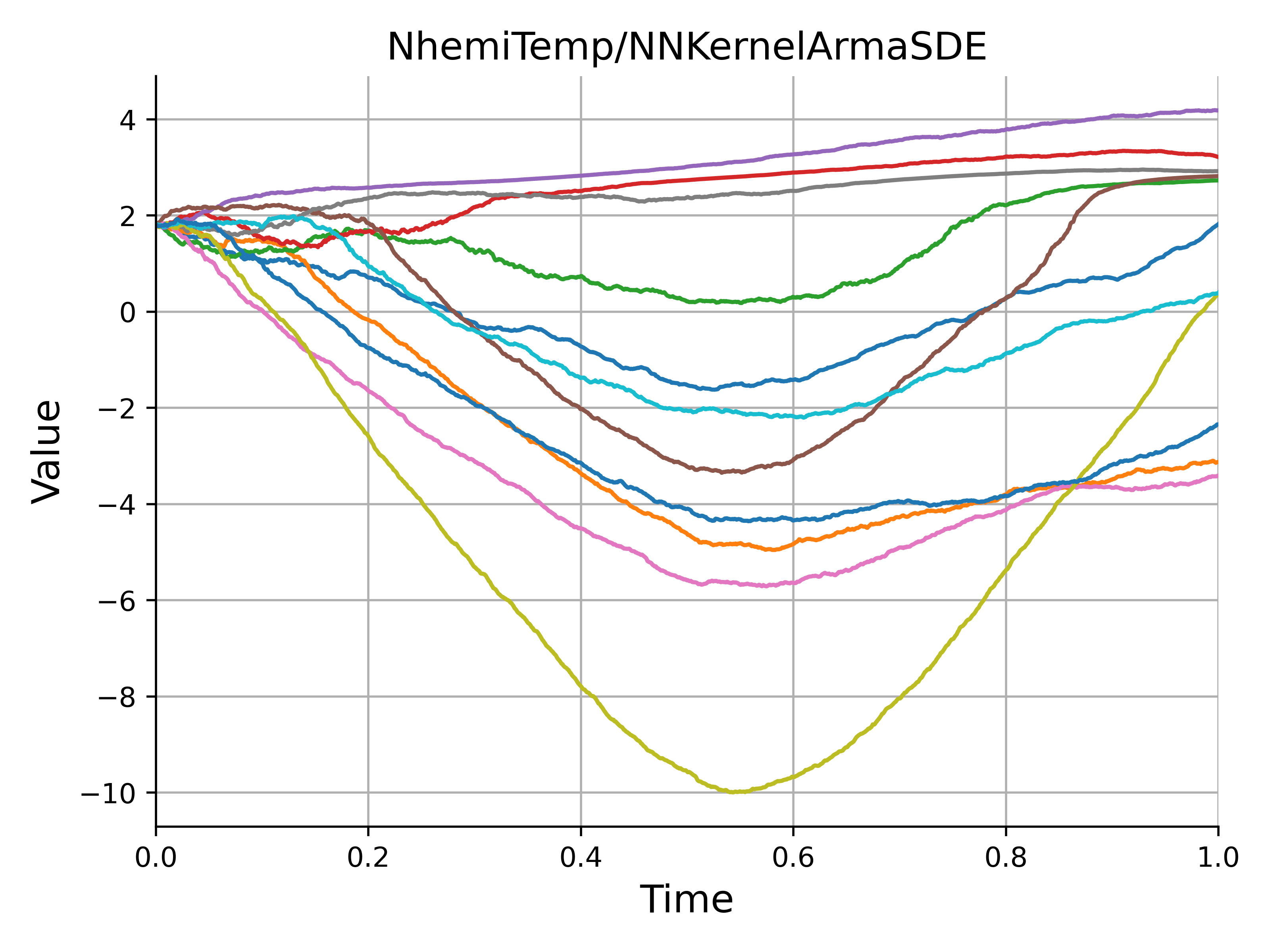}
\end{minipage}
\caption{Figure showing the paths of the original data and generated data for NhemiTemp~(Data showing long-memory). Synthetic time series are generated after calibration by RNN (upper center), SDE-Net (upper right), fSDE-Net with $H>1/2$ (lower left) and NANSDE-Net(lower center). The NANSDE-Net model successfully reproduces the characteristics of long-memory in NhemiTemp data.}
\label{fig:NhemiTemp_path}
\end{figure}

\begin{figure}[t]
\begin{minipage}[b]{0.24\linewidth}
    \centering
    \includegraphics[keepaspectratio, scale=0.2]{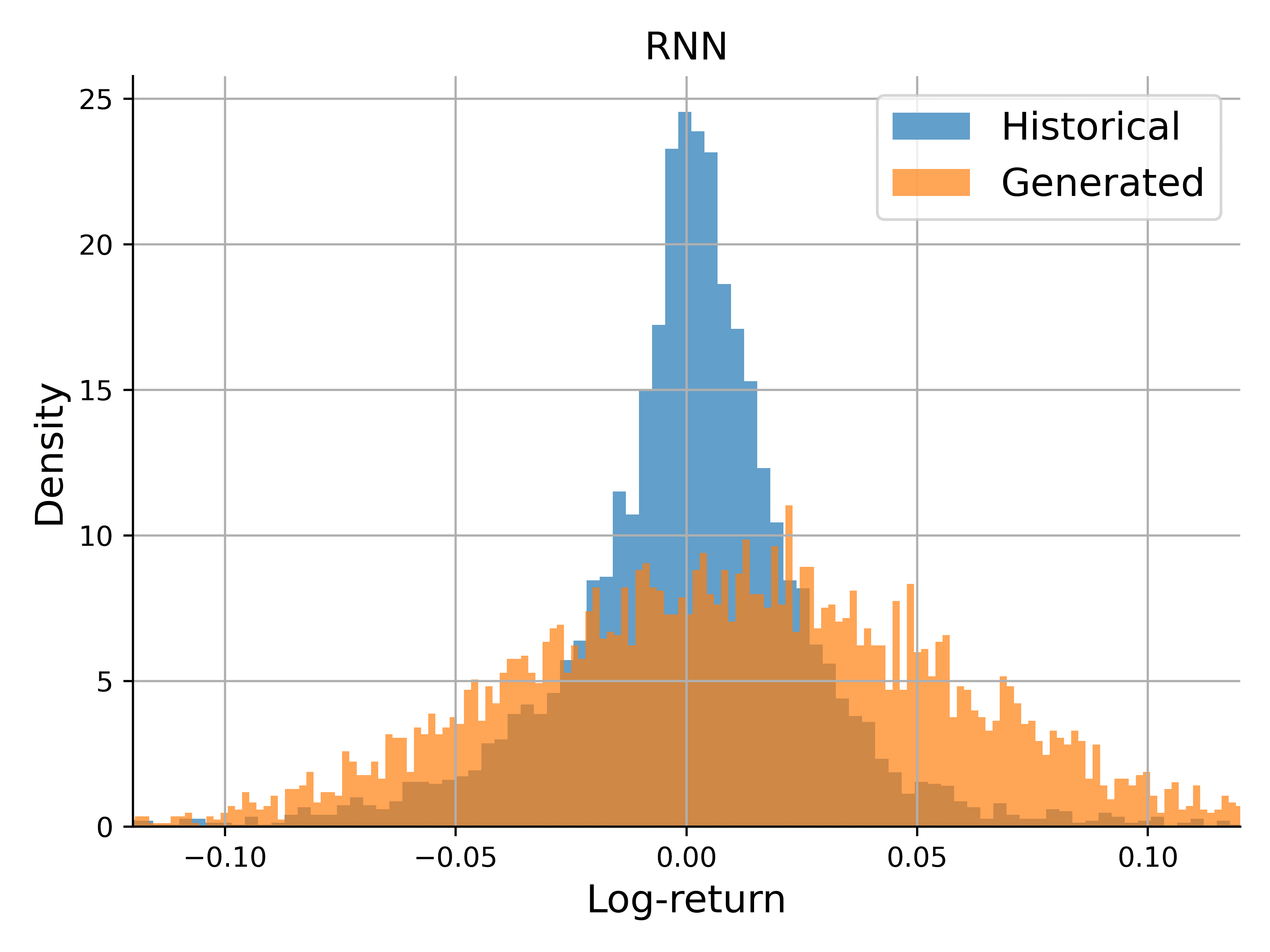}
    %\subcaption{}
\end{minipage}
%\begin{minipage}[b]{0.32\linewidth}
%    \centering
%    \includegraphics[keepaspectratio, scale=0.27]{figures/SPX_correlogram_RNN.png}
    %\subcaption{}
%\end{minipage}
\begin{minipage}[b]{0.24\linewidth}
    \centering
    \includegraphics[keepaspectratio, scale=0.2]{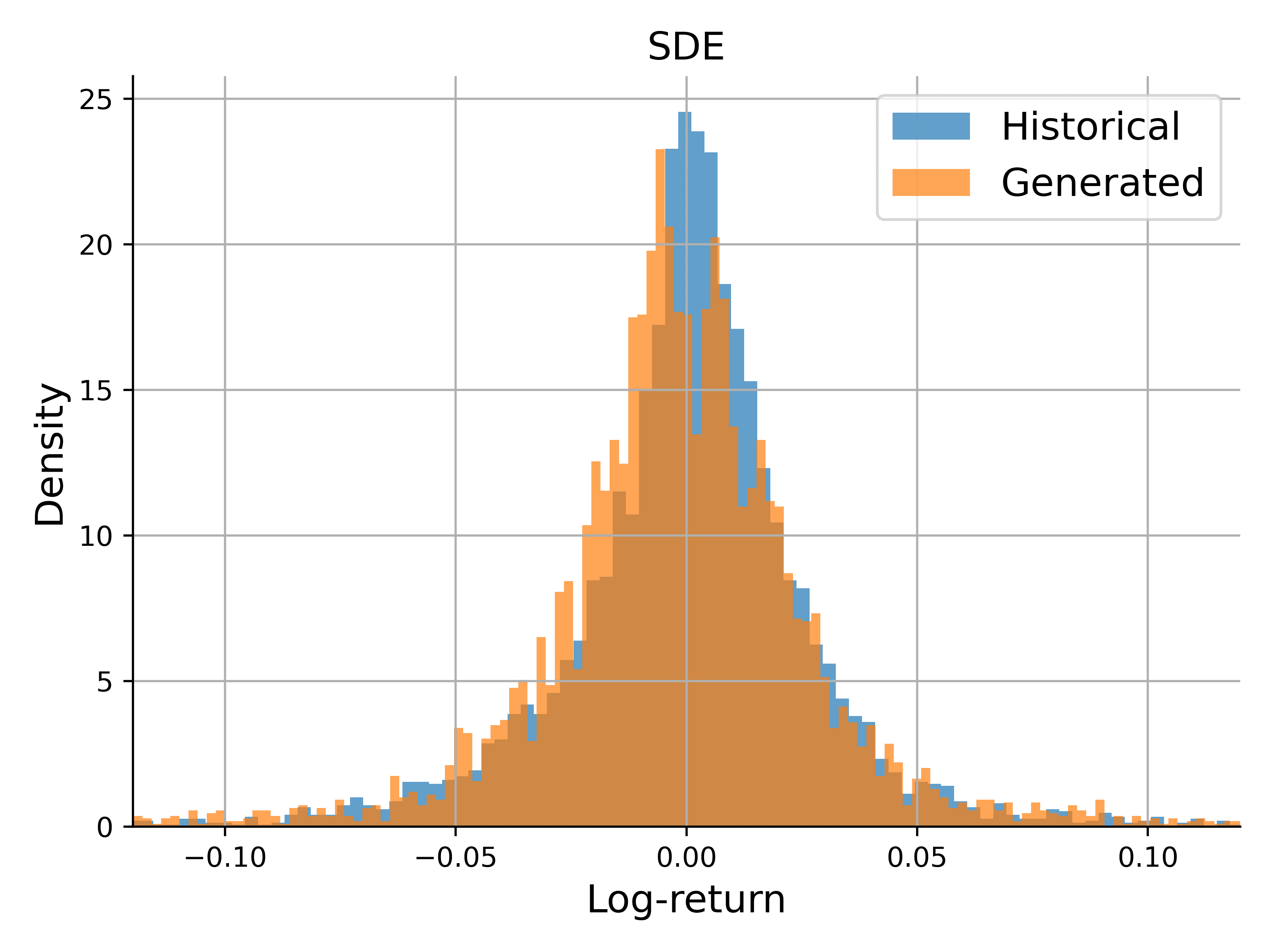}
    %\subcaption{}
\end{minipage} 
%\begin{minipage}[b]{0.32\linewidth}
%    \centering
%    \includegraphics[keepaspectratio, scale=0.27]{figures/SPX_correlogram_SDE.png}
    %\subcaption{}
%\end{minipage}
\begin{minipage}[b]{0.24\linewidth}
    \centering
    \includegraphics[keepaspectratio, scale=0.2]{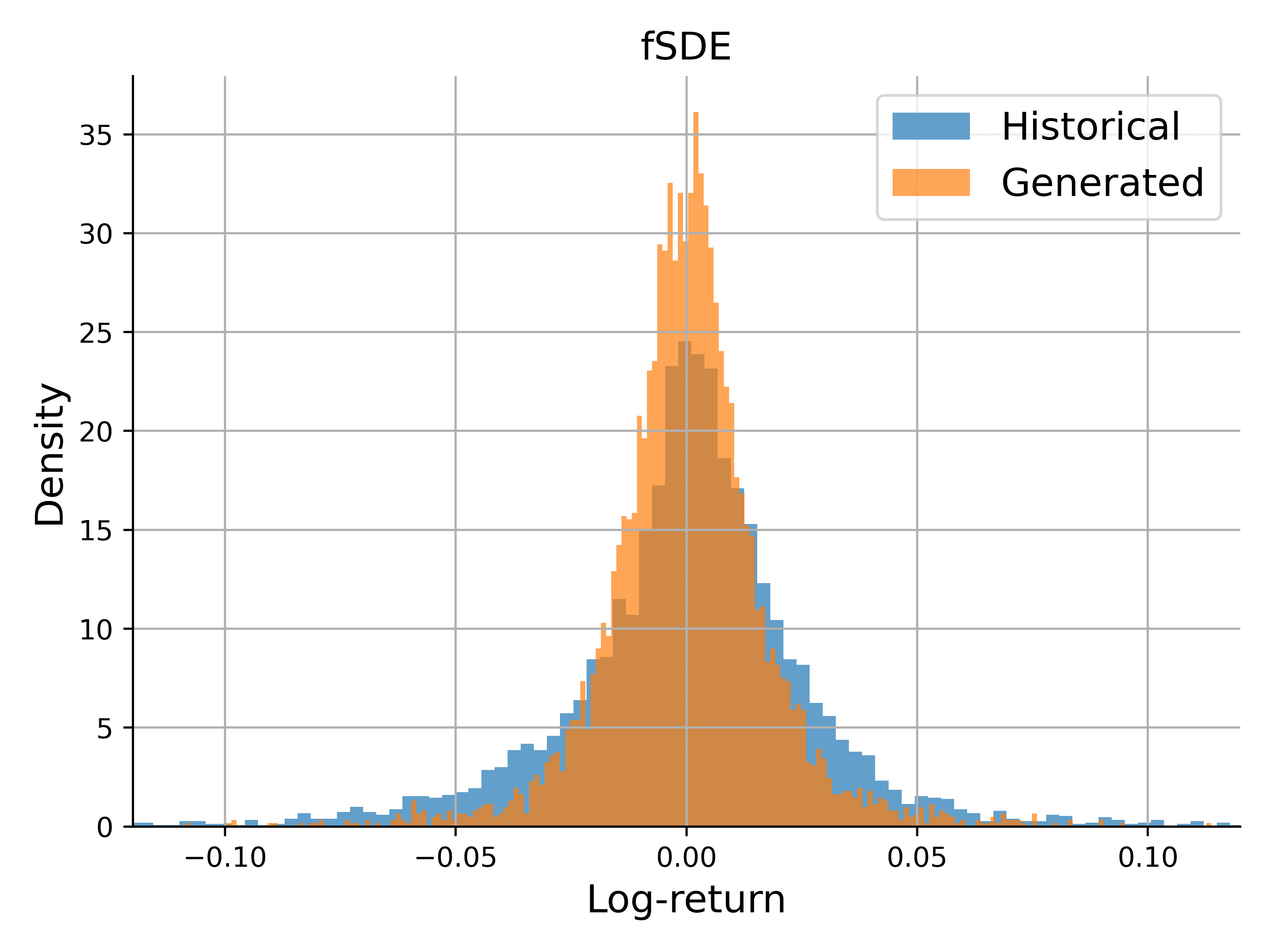}
\end{minipage}
%\begin{minipage}[b]{0.32\linewidth}
%    \centering
%    \includegraphics[keepaspectratio, scale=0.27]{figures/SPX_correlogram_fSDE.png}
    %\subcaption{}
%\end{minipage}
\begin{minipage}[b]{0.24\linewidth}
    \centering
    \includegraphics[keepaspectratio, scale=0.2]{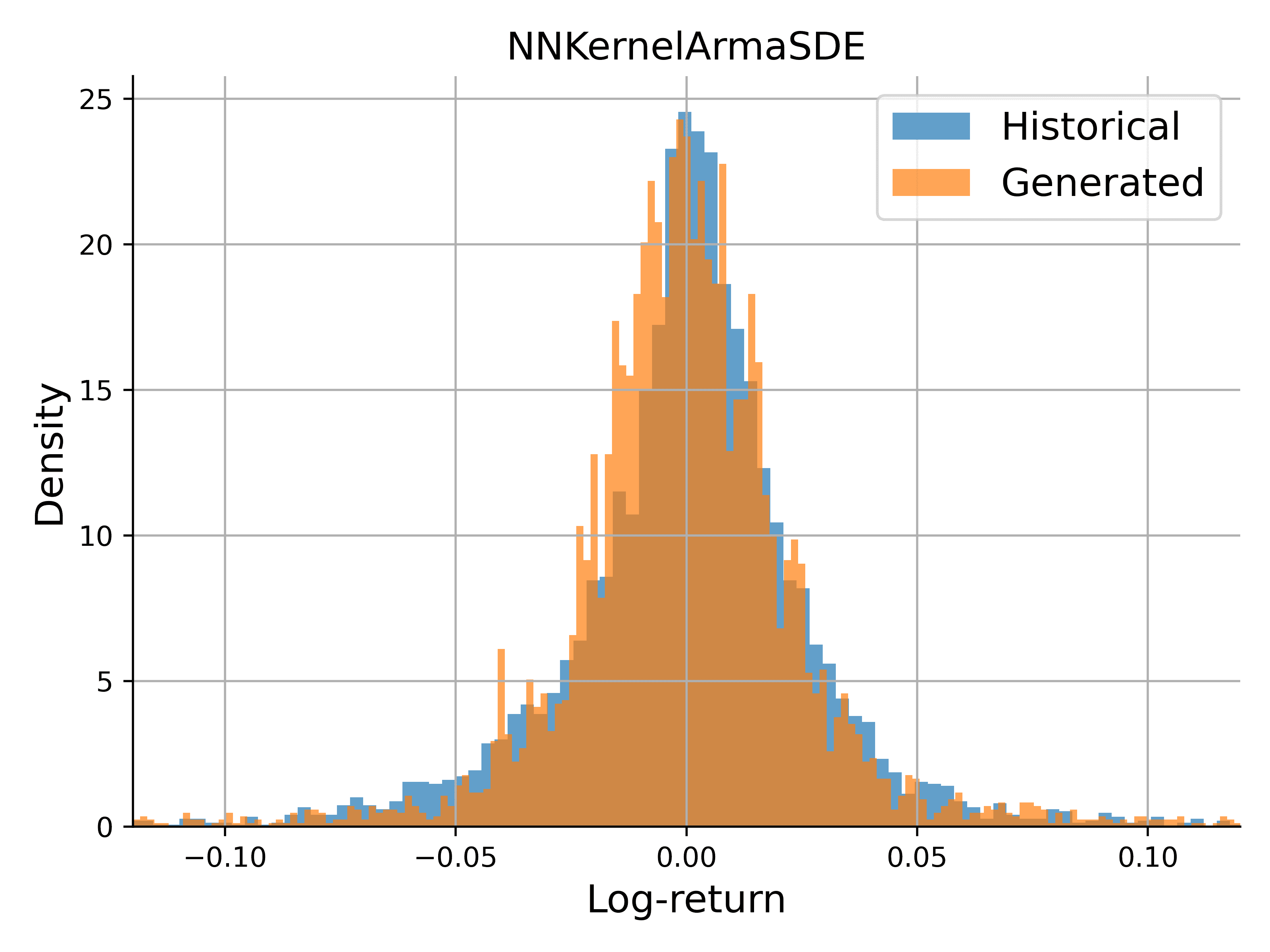}
\end{minipage}
%\begin{minipage}[b]{0.32\linewidth}
%    \centering
%    \includegraphics[keepaspectratio, scale=0.27]{figures/SPX_correlogram_NNKernelArmaSDE.png}
    %\subcaption{}
%\end{minipage}
%\caption{Comparison of the histogram (left column) and correlogram (right column) of the log return process for both synthetic and actual S\&P 500 index data. The synthetic time series are generated after calibration by RNN, SDE-Net, fSDE-Net, and NANSDE-Net, from top to bottom, respectively.}
\caption{Comparison of the histogram of the log return process for both synthetic and actual S\&P 500 index data. The synthetic time series are generated after calibration by RNN, SDE-Net, fSDE-Net, and NANSDE-Net, from left to right, respectively.}
\label{fig:SPX_histgram_correlogram}
\end{figure}

%\begin{figure}[t]
%\begin{minipage}[b]{0.32\linewidth}
%    \centering
%    \includegraphics[keepaspectratio, scale=0.27]{figures/fBM_H0.2_path_original.png}
    %\subcaption{}
%\end{minipage}
%\begin{minipage}[b]{0.32\linewidth}
%    \centering
%    \includegraphics[keepaspectratio, scale=0.27]{figures/fBM_H0.2_path_RNN.png}
%    %\subcaption{}
%\end{minipage}
%\begin{minipage}[b]{0.32\linewidth}
%    \centering
%    \includegraphics[keepaspectratio, scale=0.27]{figures/fBM_H0.2_path_SDE.png}
%\end{minipage}
%\begin{minipage}[b]{0.32\linewidth}
%    \centering
%    \includegraphics[keepaspectratio, scale=0.27]{figures/fBM_H0.2_path_fSDE.png}
%\end{minipage}
%\begin{minipage}[b]{0.32\linewidth}
%    \centering
%    \includegraphics[keepaspectratio, scale=0.27]{figures/fBM_H0.2_path_NNKernelArmaSDE.png}
%\end{minipage}
%\caption{This figure shows the paths of the original and generated data for fBm with H = 0.2 (data showing a small herst index). Synthetic time series are generated after calibration by RNN (upper right), SDE-Net (middle left), fSDE-Net with $H>1/2$ (middle right) and NANSDE-Net(lower left). It cannot be said that any of the methods accurately reproduce the small Hurst index path.}
%\label{fig:fBmH=0.2_path}
%\end{figure}

\subsection{Performance Metrics}
We prioritize the \textbf{Hurst index} as the primary measure of memory fidelity, reflecting our goal of capturing long- and short-memory behavior.

As secondary checks, we report:

\textbf{(i) Marginal distribution (total variation).}
Let $\{B_k\}_{k=1}^K$ be fixed bins. Define empirical frequencies
\[
p_k=\frac{1}{T}\sum_{t=1}^T \mathbf{1}\{r_t\in B_k\},\qquad
\hat p_k=\frac{1}{MT}\sum_{i=1}^M\sum_{t=1}^T \mathbf{1}\{\hat r^{(i)}_t\in B_k\}.
\]
The discrepancy is
\[
\mathrm{TV}=\frac{1}{2}\sum_{k=1}^K |p_k-\hat p_k|\in[0,1]\quad(0=\text{identical},\ 1=\text{maximally different}).
\]

\textbf{(ii) Time dependence (ACF scores).}
Write $r_{0:T}=(r_1,\ldots,r_T)$ and let $\gamma(\tau)=\mathrm{Corr}(|r_t|,|r_{t+\tau}|)$ for lags $\tau=1,\ldots,S$.
Set $C(r_{0:T})=(\gamma(1),\ldots,\gamma(S))\in[-1,1]^S$ and define the score
\[
\mathrm{ACF}=\Big\|\,C(r_{0:T})-\frac{1}{M}\sum_{i=1}^M C(\hat r^{(i)}_{0:T})\,\Big\|_2,
\]
where $\|\cdot\|_2$ is the Euclidean norm. A weighted variant emphasizes long lags:
\[
\mathrm{ACF}_w=\Big\|\,C(r_{0:T})\circ w-\frac{1}{M}\sum_{i=1}^M C(\hat r^{(i)}_{0:T})\circ w\,\Big\|_2,
\]
with Hadamard product $\circ$ and weights $w_\tau=2\tau/(S+1)$ (unit-mean).

\textbf{(iii) Predictive accuracy ($R^2$).}
Using an 80\%/20\% train/test split, let $\tilde r_t$ be one-step-ahead predictions on the test set $\mathcal{T}_{\text{test}}$ and
$\bar r_{\text{test}}$ its mean. Report
\[
R^2=1-\frac{\sum_{t\in\mathcal{T}_{\text{test}}}(r_t-\tilde r_t)^2}{\sum_{t\in\mathcal{T}_{\text{test}}}(r_t-\bar r_{\text{test}})^2}.
\]

For fair comparison, we follow the metric definitions and settings in \cite{hayashi2022fsde-net}.

\subsection{Results}
Table \ref{tab:result} shows results for the above performance metrics after 1000 iteration steps with 200 early stops where Adam with learning rate 0.004 is used as the optimizer. 
Examining the results for data with high Hurst indices, such as NileMin, ethernetTraffic, NBSdiff and NhemiTemp, it can be seen from the table that NANSDE-Net achieves results that are equivalent to or better than those of fSDE-Net. This suggests that NANSDE-Net can reproduce noise exhibiting long-memory with $H > 1/2$. 
As shown in Figure \ref{fig:NhemiTemp_path}, for NhemiTemp paths with long-memory, NANSDE-Net performs as well as or better than fSDE-Net. Moreover, unlike SDE-Net, it more accurately captures the characteristic features of the paths. This indicates that NA-noise, unlike Brownian motion, possesses the ability to capture long-memory characteristics.
Moreover, for datasets such as SPX, NileMin, and Ethernet, NANSDE-Net generates superior marginal distributions compared to SDE-Net.
However, from an overall perspective, focusing on performance metrics related to marginal distributions and autocorrelation functions (ACF), it cannot be conclusively stated that NANSDE-Net clearly outperforms the other methods \ref{fig:SPX_histgram_correlogram}.
Furthermore, looking at the results for data with small Hurst indices such as fBm, TPX and SX5E, RNN shows good results overall, but NANSDE-Net shows better results than other neural SDE methods such as SDE-Net and fSDE-Net. 
This demonstrates the superiority of NA-noise over other types of noise in data with $H < 1/2$. This shows that NANSDE-Net can reproduce not only $H > 1/2$ data but also $H < 1/2$ data. 
%However, it can be seen that NANSDE-Net does not successfully reproduce the roughness of data with a Hurst index of 0.2 (See the path in Figure \ref{fig:fBmH=0.2_path}). 
%Moreover, for data with $H<1/2$, SDE-Net also achieves sufficiently good results, and thus NANSDE-Net cannot be said to perform significantly better.
However, for data with $H<1/2$, SDE-Net also achieves sufficiently good results, and thus NANSDE-Net cannot be said to perform significantly better.

%%%%%%%%%%%%%%%%%%
%   Conclusion   %
%%%%%%%%%%%%%%%%%%
\section{Conclusion}\label{sec_con}
Our contributions are as follows:
\begin{itemize}
\item We introduce NANSDE-Net, an extension of SDE-Net that incorporates NN-kernel ARMA-type noise.
The proposed approach is based on the theoretical results presented in \cite{inoue2005noise1} and \cite{inoue2005noise2}.
\item We prove that the solution of NANSDE exists uniquely under standard Lipschitz and growth conditions, and derive a backpropagation formula for efficient gradient-based training.
\item The experimental results demonstrate the ability of NANSDE-Net to accurately reproduce the estimate of the Hurst index and the distribution characteristics of the original time series, particularly in the context of long-memory data.
Our code is available at \url{https://github.com/ozhr/ArmaNoise-SDE-Net}.
\end{itemize}

\subsection*{Further Study}

Since our primary focus is on investigating the role of the noise process in capturing memory characteristics, we deliberately adopt a simple neural network architecture, without attempting to optimize the network design or improve training efficiency. 
Accordingly, the experimental performance could potentially be enhanced by employing more advanced network architectures or by designing more effective loss functions.

As discussed in Section~\ref{sec_Prelim}, kernels of the form $a(t)=\frac{p}{(t+1)^{p+1}}$ satisfy the conditions for long-memory. 
However, the explicit expression of the corresponding function $\ell(s,u)$ remains unknown. 
Identifying this function constitutes an important direction for future research. 
In the present study, we instead approximate $\ell(s,u)$ by decomposing it as $\ell(s,u) = \ell_1(s)\ell_2(u)$ and learning each component via neural networks. 
Nevertheless, this decomposition does not ensure that the resulting process $Z(t)$ has stationary increments. 
Therefore, developing neural architectures that can enforce the stationary increment property remains a critical challenge. 
If such stationarity is guaranteed, it becomes possible to construct a noise process that is Gaussian, possesses stationary increments, and exhibits memory effects—thus offering a viable alternative to fBm.

Moreover, we observed that NANSDE-Net did not perform well on datasets characterized by very low Hurst indices. 
Addressing this issue requires a more comprehensive investigation into the general relationship between ARMA-type noise kernels $a(\cdot)$ and the Hurst index. 
The ultimate goal of this line of research is to identify kernel structures that enable the generation of sample paths, even for time series with low regularity.

\subsection*{Broader Impact}
Fractional Brownian motion (fBm) has traditionally been employed as the primary noise source with memory for modeling time series data exhibiting long- and short-memory behavior across various domains.
However, the findings of this study reveal that alternative noise processes can also capture such memory characteristics. This raises a new question as to whether fBm is indeed the most suitable choice for modeling long- and short-memory phenomena.
In particular, by appropriately designing the kernel function $\ell(s, u)$ to satisfy stationarity conditions, we demonstrate the existence of stationary-increment Gaussian processes with memory that are distinct from fBm.

\begin{credits}
\subsubsection{\ackname}
This work was supported by JST BOOST, Japan Grant Number JPMJBS2424.

\subsubsection{\discintname}
The authors have no competing interests to declare that are relevant to the content of this article.
\end{credits}

\end{document}